\documentclass{article}

\usepackage{arxiv}
\usepackage[ruled,vlined]{algorithm2e}
\usepackage{amsmath}
\usepackage{amsthm}
\usepackage{amssymb}
\usepackage{makecell}
\usepackage{caption}
\usepackage{booktabs}
\usepackage{hyperref}
\usepackage{accents}

\hypersetup{colorlinks,linkcolor={blue},citecolor={blue},urlcolor={red}} 

\newtheorem{theorem}{Theorem}

\newtheorem{lemma}{Lemma}
\newtheorem*{lemma*}{Lemma}

\theoremstyle{definition}
\newtheorem*{definition*}{Definition}

\newtheorem{definition}{Definition}

\theoremstyle{assumption}

\newtheorem{remark}{Remark}
\usepackage[round]{natbib}

\title{Minimax Sample Complexity for Turn-based Stochastic Game}

\author{
Qiwen Cui \\
  School of Mathematical Science\\
  Peking University\\
  \texttt{cuiqiwen@pku.edu.cn} \\
   \And
 Lin F. Yang\\
   Electrical and Computer Engineering Department,\\
  University of California, Los Angles\\
  \texttt{linyang@ee.ucla.edu}
}

\begin{document}

%

%

\maketitle

\begin{abstract}
 The empirical success of Multi-agent reinforcement learning is encouraging, while few theoretical guarantees have been revealed. In this work, we prove that the plug-in solver approach, probably the most natural reinforcement learning algorithm, achieves minimax sample complexity for turn-based stochastic game (TBSG). Specifically, we plan in an empirical TBSG by utilizing a `simulator' that allows sampling from arbitrary state-action pair. We show that the empirical Nash equilibrium strategy is an approximate Nash equilibrium strategy in the true TBSG and give both problem-dependent and problem-independent bound. We develop absorbing TBSG and reward perturbation techniques to tackle the complex statistical dependence. The key idea is artificially introducing a suboptimality gap in TBSG and then the Nash equilibrium strategy lies in a finite set. 
\end{abstract}

\section{Introduction}

Reinforcement learning (RL) \citep{sutton2018reinforcement}, a framework where agent learns to make sequential decisions in an unknown environment, has received tremendous attention. An interesting branch is multi-agent reinforcement learning (MARL) that multiple agents exist and they interact with the environment as well as the others, which bridges RL and game theory. In general, each agent attempts to maximize its own reward by analyzing the data collected from the environment and also inferring other agents' strategies. Impressive successes have been achieved in games such as backgammon \citep{tesauro1995temporal}, Go \citep{silver2017mastering} and strategy games \citep{ye2020mastering}. MARL has shown the potential for superhuman performance, but theoretical guarantees are rather rare due to complex interaction between agents that makes the problem considerably harder than single agent reinforcement learning. This is also known as non-stationarity in MARL, which means when multiple agents alter their strategies based on samples collected from previous strategy, the system becomes non-stationary for each agent and the improvement can not be guaranteed. One fundamental question in MBRL is how to design efficient algorithms to overcome non-stationarity.

Two-players turn-based stochastic game (TBSG) is a two-agent generalization of Markov decision process (MDP), where two agents choose actions in turn and one agent wants to maximize the total reward while the other wants to minimize it. As a zero-sum game, TBSG is known to have Nash equilibrium strategy \citep{shapley1953stochastic}, which means there exist a strategy pair that both agent will not benefit from changing its strategy alone, and our target is to find the (approximate) Nash equilibrium strategy. Dynamic programming type algorithms is a basic but powerful approach to solve TBSG. Strategy iteration, a counterpart of policy iteration in MDP, is known to be a polynomial complexity algorithm to solve TBSG with known transition kernel \citep{hansen2013strategy,jia2020towards}. However, these algorithms suffer from high computational cost and require full knowledge of the transition dynamic. Reinforcement learning is a promising alternative, which has demonstrated its potential in solving MDPs like Atari game. However, non-stationarity pose a large obstacle against the convergence of model-free algorithms and sophisticated algorithms have been proposed to tackle this challenge \citep{jia2019feature,sidford2020solving}. In this work, we focus on another promising but insufficiently developed method, model-based algorithms, where agents learn the model of the environment and plan in the empirical model. 

Model-based RL is long perceived to be the cure to the sample inefficiency in RL, which is also justified by empirical advances \citep{kaiser2019model,wang2019benchmarking}. However, the theoretical understanding of model-based RL is still far from complete. Recently, a line of research focuses on analyzing model-based algorithm under generative model setting \cite{azar2013minimax,agarwal2019optimality,cui2020plugin,zhang2020model,li2020breaking}, In this work, we aim to prove that the simplest model-based algorithm, plug-in solver approach, enjoys minimax sample complexity for TBSG by utilizing novel absorbing TBSG and reward perturbation techniques.

Specifically, we assume that we have access to a generative model oracle \citep{kakade2003sample}, which allows sampling from arbitrary state-action pair. It is well known that exploration and exploitation tradeoff is simplified under this setting, i.e., sampling from each state-action pair equally can already yield the minimax sample complexity result. With the generative model, the most intuitive approach is to learn an empirical TBSG and then use a planning algorithm to find the empirical Nash equilibrium strategy as the solution, which separates learning and planning. This kind of algorithm is known as `plug-in solver approach', which means $\emph{arbitrary}$ planning algorithm can be used. We will show that this simple plug-in solver approach enjoys minimax sample complexity. For MDP, this line of research has been well understood \citep{azar2013minimax} and the recent work \citep{li2020breaking} almost completely solves this problem by proving the minimax complexity with full range of $\epsilon$. However, the result for TBSG is still lack and largely due to interaction between agents. In this work, we give the sample complexity upper bounds of TBSG for both problem-dependent and problem-independent cases, where the term `problem-dependent' means a suboptimality gap exists.

Suboptimality gap is a widely studied notion in bandit theory, which stands for a constant gap between the optimal arm and second optimal arm. This notion has received increasing focus in MDP setting and we generalize it to TBSG. To start with, we prove that $\widetilde{O}(|\mathcal{S}||\mathcal{A}|(1-\gamma)^{-3}\Delta^{-2})$ samples are enough to recover Nash equilibrium strategy accurately for $\Delta\in(0,(1-\gamma)^{-\frac{1}{2}}]$, where $\mathcal{S}$ is the state space, $\mathcal{A}$ is the action space, $\gamma$ is the discount factor and $\Delta$ is the suboptimality gap. The key analysis tool is the absorbing TBSG technique that helps to show empirical optimal Q value is close to true optimal Q value. With the absorbing TBSG, the statistical dependence on $\widehat{P}(s,a)$ is moved to a single parameter $u$, which can be approximated by an $\epsilon$-cover. Therefore, standard concentration arguments combined with union bound can be applied. Suboptimality gap plays an critical role in showing the empirical Nash equilibrium strategy is exactly the same as true Nash equilibrium strategy.

The main contribution is in the second part, where we give our problem-independent bound which meets the existing lower bound $O(|\mathcal{S}||\mathcal{A}|(1-\gamma)^{-3}\epsilon^{-2})$ \citep{azar2013minimax,sidford2020solving}. Note that the problem-dependent result becomes meaningless when the suboptimality gap is sufficiently small and also the gap may not even exist. In the problem-independence case, we develop the reward perturbation technique to create such a gap, which is inspired by the technique developed in \citep{li2020breaking}. The key observation is that if $r(s,a)$ increases, $Q^*(s,a)$ increases much faster than $Q^*(s,a')$, thus the perturbed TBSG enjoys a suboptimality gap with high probability. Different from the usage in the first part, here suboptimality gap is used to ensure empirical Nash equilibrium strategy lies in a finite set so that union bound can be applied. Combining the reward perturbation technique with the absorbing TBSG technique, we are able to prove more subtle concentration arguments and finally show that the empirical Nash equilibrium strategy is an $\epsilon$-approximate Nash equilibrium strategy in true TBSG with minimax sample complexity $\widetilde{O}(|\mathcal{S}||\mathcal{A}|(1-\gamma)^{-3}\epsilon^{-2})$ for $\epsilon\in(0,(1-\gamma)^{-1}]$. In addition, we can recover the problem-dependent bound with full range of $\epsilon$ as well.

Recently \citep{zhang2020model} proposes a similar result for simultaneous stochastic game. However, their result requires and a planning oracle for regularized stochastic game, which is computationally intractable. Our algorithm only require planning in a standard turn-based stochastic game, which can be performed efficiently by utilizing strategy iteration \citep{hansen2013strategy} or learning algorithms \citep{sidford2020solving}. In addition, their sample complexity result $N=\widetilde{O}(|\mathcal{S}||\mathcal{A}|(1-\gamma)^{-3}\epsilon^{-2})$ only holds for $\epsilon\in(0,(1-\gamma)^{-\frac{1}{2}}]$, which means $N=\widetilde{O}(|\mathcal{S}||\mathcal{A}|(1-\gamma)^{-2})$. our result fills the blank in the sample region $\widetilde{O}(|\mathcal{S}||\mathcal{A}|(1-\gamma)^{-1})\leq N\leq \widetilde{O}(|\mathcal{S}||\mathcal{A}|(1-\gamma)^{-2})$. Note that the lower bound \citep{azar2013minimax} indicates that $N=O(|\mathcal{S}||\mathcal{A}|(1-\gamma)^{-1})$ is insufficient to learn a policy. As both parts of our analysis heavily rely on the suboptimality gap, we hope our work can provide more understanding about this notion and TBSG. 

\section{Preliminary}
\paragraph{Turn-based Stochastic Game}
Turn-based two-player zero-sum stochastic game (TBSG) is a generalized version of Markov decison process (MDP) which includes two players competing with each other. Player 1 aims to maximize the total reward while player 2 aims to minimize it. TBSG is described by the tuple $\mathcal{G}=(\mathcal{S}=\mathcal{S}_{\mathrm{max}}\cup\mathcal{S}_{\mathrm{min}},\mathcal{A},P,r,\gamma)$, where $\mathcal{S}_{\mathrm{max}}$ is the state space of player 1, $\mathcal{S}_{\mathrm{min}}$ is the state space of player 2, $\mathcal{A}$ is the action space of both players, $P\in\mathbb{R}^{|\mathcal{S}||\mathcal{A}|\times|\mathcal{S}|}$ is the transition probability matrix,  $r\in\mathbb{R}^{|\mathcal{S}||\mathcal{A}|}$ is the reward vector and $\gamma$ is the discount factor. In each step, only one player plays an action and a transition happens. For instance, if the state $s\in\mathcal{S}_{\mathrm{max}}$, player 1 needs to select an action $a\in\mathcal{A}$. After selecting the action, the state will transit to $s'\in\mathcal{S}_{\mathrm{min}}$ according to the distribution $P(\cdot|s,a)$ with reward $r(s,a)$ and player 2 needs to choose the action. For representation simplicity and without loss of generality, we assume that $r$ is known and only $P$ is unknown\footnote{Our proof can be easily adapted to show that the sample complexity of learning the reward $r$ is an order of $\frac{1}{1-\gamma}$ smaller than learning the transition $P$.}. 

We denote a strategy pair as $\pi:=(\mu,\nu)$, where $\mu$ is the strategy of player 1 and $\nu$ is the strategy of player 2. Given strategy $\pi$, the value function and Q-function can be defined similarly as in MDP:

$$V^\pi(s):=\mathbb{E}\left[\sum_{t=0}^{\infty}\gamma^tr(s^t,\pi(s^t))~\bigg|~ s^0=s\right],$$
\begin{align*}
    &Q^\pi(s,a):= \mathbb{E}\left[r(s^0,a^0)+\sum_{t=1}^{\infty}\gamma^tr(s^t,\pi(s^t))~\bigg|~s^0=s,a^0=a\right]=r(s,a)+\gamma P(s,a)V^\pi.
\end{align*}

From the perspective of player 1, if the strategy $\nu$ of player 2 is given, TBSG degenerates to an MDP, so the optimal policy against $\nu$ exists, which we called as counterstrategy and use $c_{\mathrm{max}}(\nu)$ to denote it.
Similarly we can define $c_{\mathrm{min}}(\mu)$ as the counterstrategy of $\mu$ for player 2. For simplicity, we ignore the subscript in $c_{\mathrm{max}}$ and $c_{\mathrm{min}}$ when it is clear in the context. In addition, we define $V^{*,\nu}:=V^{c(\nu),\nu}$ and $V^{\mu,*}:=V^{\mu,c(\mu)}$ and the same for $Q$. By definition and property of optimal policy in MDP, we have
$$Q^{*,\nu}(s,a)=\max_{\mu}Q^{\mu,\nu}(s,a),\forall s\in\mathcal{S},$$
$$Q^{\mu,*}(s,a)=\min_{\nu}Q^{\mu,\nu}(s,a),\forall s\in\mathcal{S},$$
$$Q^{*,\nu}(s,c_{\mathrm{max}}(\nu)(a))=\max_{a'}Q^{*,\nu}(s,a'),\forall s\in\mathcal{S}_\mathrm{max},$$
$$Q^{\mu,*}(s,c_{\mathrm{min}}(\mu)(s))=\min_{a'}Q^{\mu,*}(s,a'),\forall s\in\mathcal{S}_\mathrm{min}.$$
Note that these are the sufficient and necessary condition of counterstrategy, which will be utilized repeatedly in our analysis.

To solve a TBSG, our goal is to find the Nash equilibrium strategy $\pi^*=(\mu^*,\nu^*)$, where $\mu^*=c(\nu^*),\nu^*=c(\mu^*)$. For Nash equilibrium strategy, neither player can benefit from changing its policy alone. As $\mu^*$ and $\nu^*$ are counterstrategy to each other, they inherit properties of counterstrategy given above. For simplicity, we will not repeat here. It is well known that in TBSG, there always exists pure strategy as Nash equilibrium strategy. In addition, all Nash equilibrium strategy share the same state-action value, which makes pure Nash equilibrium strategy unique given some tie selection rule, so we only consider pure strategy in our analysis.

Specifically, our target is to find an $\epsilon$-approximate Nash equilibrium strategy $\pi=(\mu,\nu)$ such that $$\big|Q^{\mu,*}(s,a)-Q^*(s,a)\big|\leq\epsilon,\forall (s,a),$$ $$\big|Q^{*,\nu}(s,a)-Q^*(s,a)\big|\leq\epsilon,\forall (s,a),$$ for some $\epsilon>0$ with as few samples as possible. Note that this is different from and stronger than the MDP analogue, which should be $\big|Q^{\pi}(s,a)-Q^*(s,a)\big|\leq\epsilon$. This slight difference makes subtle difficulty as we will show later.

\paragraph{Generative Model Oracle and Plug-in Solver Approach}
We assume that we have access to a generative model, where we can input arbitrary state action pair $(s,a)$ and receive a sampled state form $P(\cdot|s,a)$. Generative model oracle was introduced in \citep{kearns1999finite,kakade2003sample}. This setting is different from the offline oracle where we can only sample trajectories via a behaviour policy and online oracle where we adaptively change the policy to explore and exploit. The advantage of generative model setting is that the exploration and exploitation is simplified, as previous work shows that treating all state-action pair equally is already optimal \citep{azar2013minimax,sidford2018variance}. In particular, we call the generative model $N/|\mathcal{S}||\mathcal{A}|$ times on each state-action pair and construct the empirical TBSG $\widehat{G}=(\mathcal{S}=\mathcal{S}_{\mathrm{max}}\cup\mathcal{S}_{\mathrm{min}},\mathcal{A},\widehat{P},r,\gamma)$:
$$\widehat{P}(s'|s,a)=\frac{\mathrm{count}(s,a,s')}{N/|\mathcal{S}||\mathcal{A}|},\forall s,s'\in\mathcal{S}_{\mathrm{max}}\cup\mathcal{S}_{\mathrm{min}},a\in\mathcal{A},$$
where $\mathrm{count}(s,a,s')$ is the number of times that $s'$ is sampled from state-action pair $(s,a)$.
It is straightforward that $\widehat{P}$ is an unbiased and maximum likelihood estimation of the true transition kernel $P$. We use $\widehat{\pi}^*=(\widehat{\mu}^*,\widehat{\nu}^*)$ to denote the Nash equilibrium strategy in the empirical MDP as well as $\widehat{V}$ and $\widehat{Q}$. The algorithm is given in Algorithm \ref{alg1}.

As the transition kernel in $\widehat{\mathcal{G}}$ is known, arbitrary planning algorithm can be used to find the empirical Nash equilibrium strategy $\widehat{\pi}^*$. One choice is to use strategy iteration, which finds $\widehat{\pi}^*$ with in $\widetilde{O}(|\mathcal{S}||\mathcal{A}|(1-\gamma)^{-1})$ iterations \citep{hansen2013strategy}. In addition, algorithms that find approximate Nash equilibrium strategy can be applied, such as QVI-MIVSS in \citep{sidford2020solving}. Note that our analysis is for the exact Nash equilibrium strategy $\widehat{\pi}^*$, but it can be generalized to approximate Nash equilibrium strategy by using techniques in \citep{agarwal2019optimality}.

\paragraph{Suboptimality Gap}

Suboptimality gap is originated in bandit theory, which is the gap between the mean reward of the best arm and second best arm. In TBSG, we define the suboptimality gap based on optimal $Q$-value and second optimal $Q$-value. 

\begin{definition}
(Suboptimality gap for Nash equilibrium strategy) A TBSG enjoys a suboptimality gap of $\Delta$ for Nash equilibrium strategy if and only if 
$$\forall s\in \mathcal{S}_{\mathrm{max}},a\neq\mu^*(s): Q^*(s,\mu^*(s))-Q^*(s,a)\geq\Delta,$$
$$\forall s\in\mathcal{S}_{\mathrm{min}},a\neq\nu^*(s): Q^*(s,\nu^*(s))-Q^*(s,a)\leq -\Delta,$$
and if there are two optimal actions, the gap is zero.
\end{definition}

\begin{definition}
(Suboptimality gap for counterstrategy) A TBSG enjoys a suboptimality gap of $\Delta$ for counter strategy to the strategy $\nu$ of player 2 if and only if 
$$\forall s\in \mathcal{S}_{\mathrm{max}},a\neq c(\nu)(s), Q^*(s,c(\nu)(s))-Q^*(s,a)\geq\Delta.$$
A TBSG enjoys a suboptimality gap of $\Delta$ for counter strategy to the strategy $\mu$ of player 1 if and only if 
$$\forall s\in \mathcal{S}_{\mathrm{min}},a\neq c(\mu)(s), Q^*(s,c(\mu)(s))-Q^*(s,a)\leq-\Delta.$$
\end{definition}

The suboptimality gap means that following the Nash equilibrium strategy, the expected total reward of the best action and the second best action differs at least $\Delta$. Intuitively, this gap quantifies the difficulty of learning the optimal action and a small gap hinders finding out the optimal action.

\paragraph{Notations}
$f(x)=O(g(x))$ means that there exists a constant $C$ such that $f\leq Cg$ and $f(x)=\Omega(g(x))$ means that $g(x)=O(f(x))$. $\widetilde{O}$ and $\widetilde{\Omega}$ is same as $O$ and $\Omega$ except that logarithmic factors are ignored. We use $f\gtrsim g$ to denote that there exist some constant $C$ such that $f\geq Cg$ and $f\lesssim g$ means $f\leq Cg$ for some constant $C$. For a strategy $\pi$, $P^\pi\in\mathbb{R}^{|\mathcal{S}||\mathcal{A}|\times|\mathcal{S}||\mathcal{A}|}$ is the transition matrix induced by policy $\pi$ and $P^\pi(s,a)(s',a')=P(s'|s,a)\mathbf{1}(a'=\pi(s'))$ where $\mathbf{1}(\cdot)$ is the indicator function. $P(s,a)$ is the row vector of $P$ that correspond to $s,a$. We use $|\cdot|$ to denote the infinity norm $\|\cdot\|_\infty$ and $\sqrt{\cdot},\leq,\geq$ are entry-wise operators. 

\begin{algorithm}[]
	\SetAlgoLined
	\KwIn{A generative model that can output samples from distribution $P(\cdot|s,a)$ for query $(s,a)$, a plug-in solver.}
	\textbf{Initial}: Sample size: $N$\;
	\For{(s,a) in $\mathcal{S}\times\mathcal{A}$}{
		Collect $N/|\mathcal{S}||\mathcal{A}|$ samples from $P(\cdot|s,a)$\;
    Compute $\widehat{P}(s'|s,a)=\frac{count(s,a,s')}{N/|\mathcal{S}||\mathcal{A}|}$\;}
    Construct the (perturbed) empirical TBSG\;
    $\widehat{\mathcal{G}}=(\mathcal{S},\mathcal{A},\widehat{P},r,\gamma)$\;
    $\widehat{\mathcal{G}}_\mathrm{p}=(\mathcal{S},\mathcal{A},\widehat{P},r_\mathrm{p},\gamma)$\;
	Plan with the plug-in solver: receive the (perturbed) empirical Nash equilibrium strategy $\widehat{\pi}^*$ ($\widehat{\pi}_\mathrm{p}^*$)\;
    \KwOut{$\widehat{\pi}^*$ ($\widehat{\pi}_\mathrm{p}^*$)}
	\caption{Solving TBSG via Plug-in Solver}
	\label{alg1}
\end{algorithm}

\section{Technical Lemmas from MDP}

In this section, we present several technical lemmas that is originated in MDP analysis \citep{azar2013minimax,agarwal2019optimality,li2020breaking}. These lemmas can be easily adapted to TBSG and we present them to give some intuition on our TBSG analysis.

\begin{lemma}
For any strategy $\pi$, we have
\begin{align*}
Q^\pi-\widehat{Q}^\pi=\gamma(I-P^\pi)^{-1}(\widehat{P}-P)\widehat{V}^\pi=\gamma(I-\widehat{P}^\pi)^{-1}(P-\widehat{P})V^\pi.    
\end{align*}
\label{l1}
\end{lemma}

This lemma portrays the concentration of $\widehat{Q}^\pi$. Note that there are two kinds of factorization and requires different analysis. In the problem-dependent bound, we use the first one and in the problem-independent bound, we use the second one.

\begin{definition}
(One-step Variance) We define the one-step variance of a state-action pair $(s,a)$ with respect to a certain value function $V$ to be
$$Var_{s,a}(V):=P(s,a)V^2-(P(s,a)V)^2,$$
which is the variance of next state value. We define $Var_P(V)\in\mathbb{R}^{|\mathcal{S}||\mathcal{A}|}$ to be a vector consisting of all $Var_{s,a}(V)$.
\end{definition}

We define the one-step variance to facilitate the usage of Bernstein's inequality on term $(\widehat{P}-P)\widehat{V}^\pi$. A detailed introduction of variance in MDP can be found in \cite{azar2013minimax}. The following two lemmas show how to bound the one-step variance term.

\begin{lemma}
	For any policy $\pi$ and $V^\pi$ is the value function in a MDP with transition $P$, we have
	$$\big|(I-\gamma P^\pi)^{-1}\sqrt{Var_P(\widehat{V}^\pi)}\big|\leq\sqrt{\frac{2}{(1-\gamma)^3}}+\frac{\big|Q^\pi-\widehat{Q}^\pi\big|}{1-\gamma}.$$
	\label{l2}
\end{lemma}

\begin{lemma}
For any policy $\pi$ and $V^\pi$ is the value function in a MDP with transition $P$, if $N\gtrsim\frac{|\mathcal{S}||\mathcal{A}|}{1-\gamma}\log\left(\frac{1}{(1-\gamma)\delta}\right)$, with probability larger than $1-\delta$, we have
	$$\big|(I-\gamma \widehat{P}^\pi)^{-1}\sqrt{Var_P(V^\pi)}\big|\leq\frac{16}{\sqrt{(1-\gamma)^3}}.$$
	\label{l3}
\end{lemma}

Lemma 2 correspond to the first factorization in Lemma 1 and Lemma 3 is correspond to the second one. Lemma 2 is derived by utilizing the variance-Bellman-equation and Lemma 3 is derived by a taylor expansion type analysis. If we can apply Bernstein's inequality to $(\widehat{P}-P)\widehat{V}^\pi$ or $(\widehat{P}-P)V^\pi$ to generate the $\sqrt{Var_P(\widehat{V}^\pi)}$ or $\sqrt{Var_P(V^\pi)}$ term, then with Lemma 2 or Lemma 3, we can bound $|Q^\pi-\widehat{Q}^\pi|$. 

By concentration inequalities, we can bound $(P-\widehat{P})V^\pi$ for \emph{fixed} $\pi$. However, if $\pi=(\widehat{\mu}^*,c(\widehat{\mu}^*))$ or $\pi=(\mu^*,\widehat{c}(\mu))$, the complex statistical dependence between $\pi$ and $\widehat{P}$ hinders the conventional concentration and subtle techniques are needed. In addition, $(P-\widehat{P})\widehat{V}^\pi$ suffers from the dependence even for fixed $\pi$. The key contribution in the next two section is to use novel TBSG techniques to make concentration arguments applicable to these two terms.

\section{Warm up: Problem-dependent Upper Bound}

In this section, we show that if a TBSG $\mathcal{G}$ enjoys a suboptimality gap of $\Delta$, then with $\widetilde{O}(|\mathcal{S}||\mathcal{A}|(1-\gamma)^{-3}\Delta^{-2})$ samples, the empirical Nash equilibrium strategy $\widehat{\pi}^*$ in $\widehat{\mathcal{G}}$ is exactly the Nash equilibrium strategy $\pi^*$ in $\mathcal{G}$ with high probability. To begin with, we introduce a novel absorbing TBSG technique, which is motivated by the absorbing MDP technique developed in \citep{agarwal2019optimality}. An absorbing TBSG $\widetilde{\mathcal{G}}_{s,a,u}$ is identical to the empirical TBSG $\widehat{\mathcal{G}}$ except that the transition distribution of a specific state-action pair $(s,a)$ is set to be absorbing. Similar techniques have been developed for MDP \citep{li2020breaking} and simultaneous game \citep{zhang2020model}.

\begin{definition} (Absorbing TBSG)
For a TBSG $\mathcal{G}=(\mathcal{S}_{max},\mathcal{S}_{min},\mathcal{A},P,r,\gamma)$ and a given state-action pair $(s,a)\in(\mathcal{S}_{max}\cup \mathcal{S}_{min})\times\mathcal{A}$, the absorbing TBSG  $\widetilde{\mathcal{G}}_{s,a,u}=(\mathcal{S}_{max},\mathcal{S}_{min},\mathcal{A},\widetilde{P},\widetilde{r},\gamma)$, where 
$$\widetilde{P}(s|s,a)=1,\widetilde{P}(\cdot|s',a')=P(\cdot|s',a'),\forall (s',a')\neq(s,a),$$
$$\widetilde{r}(s,a)=u,\widetilde{r}(s',a')=r(s',a'),\forall (s',a')\neq(s,a).$$
\end{definition}

\begin{remark}
Absorbing TBSG is independent of $\widehat{P}(s,a)$, which is a kind of leave-one-out analysis. Note that $\widetilde{P}(s,a)$ can be set to arbitrary fixed distribution. We use the absorbing distribution for simplicity and correspondence to its name.
\end{remark}

For simplicity, we ignore $(s,a)$ in absorbing TBSG when there is no misunderstanding. We use $\widetilde{\pi}_u,\widetilde{Q}_u,\widetilde{V}_u$ to denote the strategy, state-action value and state value in the absorbing TBSG. These terms actually depend on $(s,a)$ and we omit $(s,a)$ when there is no confusion. Note that $\widetilde{\mathcal{G}}_{u}$ has no dependence on $\widehat{P}(s,a)$, which makes the concentration on $(\widehat{P}(s,a)-P(s,a))\widetilde{V}_u$ possible. The key property of absorbing TBSG is that it can recover the Q-value of the empirical MDP by tuning the parameter $u$. Moreover, the Q-value of absorbing TBSG is $\frac{1}{1-\gamma}$-lipschitz to $u$, which means we can use an $\epsilon$-cover on the range of $u$ to approximately recover the Q value in the empirical TBSG.

\begin{lemma} (Properties of absorbing TBSG)
Set $u^*=r(s,a)+\gamma(P(s,a)V^*)-\gamma V^*(s),u^\mu=r(s,a)+\gamma(P(s,a)V^\mu)-\gamma V^\mu(s)$. Then we have
$$Q_{u^*}^*=Q^*,Q_{u^\mu}^{\mu,*}=Q^{\mu,*},$$
$$|Q_u^*-Q_{u'}^*|\leq\frac{|u-u'|}{1-\gamma},|Q_{u}^{\mu,*}-Q_{u'}^{\mu,*}|\leq\frac{|u-u'|}{1-\gamma}.$$
\label{l4}
\end{lemma}

 We can concentrate $(\widehat{P}(s,a)-P(s,a))\widehat{V}^{\mu,*}$ by a union bound and an additional approximation error term as we construct an absorbing TBSG $\widetilde{\mathcal{G}}$ on empirical TBSG $\widehat{\mathcal{G}}$. Combining with Lemma 1 and Lemma 2, we can bound $|Q^{\mu,\widehat{c}(\mu)}-\widehat{Q}^{\mu,c(\mu)}|$.

Now we show how to use absorbing TBSG and suboptimality gap to prove that $\widehat{\pi}^*=\pi^*$ with large probability. The core is to prove that $|Q^*-\widehat{Q}^*|\leq\frac{\Delta}{2}$, then by the definition of suboptimality gap, $\widehat{Q}^*$ enjoys the gap for policy $\pi^*$. As $\widehat{\pi}^*$ is the Nash equilibrium policy in $\widehat{\mathcal{G}}$, it is the only policy that can enjoy the gap, which means $\widehat{\pi}^*=\pi^*$.

\begin{lemma}
\begin{align*}
    \big|Q^*-\widehat{Q}^*\big|\leq\max\{\big|Q^{\mu,\widehat{c}(\mu)}-\widehat{Q}^{\mu,*}\big|,\big|Q^{\widehat{c}(\nu),\nu}-\widehat{Q}^{*,\nu}\big|\}.
\end{align*}
\label{l5}
\end{lemma}

\begin{lemma}
If $Q^*$ enjoys a suboptimality gap of $\Delta$ and $\big|Q^*-\widehat{Q}^*\big|\leq \frac{\Delta}{2}$, then we have $\widehat{\pi}^*=\pi^*.$
\label{l6}
\end{lemma}

 Combining all the parts together, we get our problem-dependent result. Note that this result can recover $\pi^*$ exactly, but at a price of restrictive sample complexity. This means we have little knowledge about if the sample complexity is slightly fewer, can we approximately recover $\pi^*$. In the next section, we will give a problem-independent result, which has no dependence on the suboptimality gap of $\mathcal{G}$. 
\begin{theorem}
If $\mathcal{G}$ enjoys a suboptimality gap of $\Delta$ and the number of samples satisfies
$$N\geq\frac{C|\mathcal{S}||\mathcal{A}|}{(1-\gamma)^3\Delta^2}\log(\frac{|\mathcal{S}||\mathcal{A}|}{(1-\gamma)\delta\Delta})$$
for some constant $C$ and $\Delta\in(0,(1-\gamma)^{-\frac{1}{2}}]$, then with probability at least $1-\delta$, we have $\widehat{\pi}^*=\pi^*$, which means the empirical Nash equilibrium strategy we obtained is exactly the Nash equilibrium strategy in the true TBSG.
\label{th1}
\end{theorem}

\section{Problem-independent upper bound}

Two flaws exist in our problem-dependent result. One is if the suboptimality gap is considerably small, the bound $\widetilde{O}(|\mathcal{S}||\mathcal{A}|(1-\gamma)^{-3}\Delta^{-2})$ becomes meaningless. In addition, when $N<|\mathcal{S}||\mathcal{A}|(1-\gamma)^{-3}\Delta^{-2}$, the quality of empirical Nash equilibrium strategy $\widehat{\pi}^*$ is completely unknown. In this section, we aim to give a minimax sample complexity result without the assumption of suboptimality gap. Interestingly, though the suboptimality gap is not assumed, we create the suboptimality gap instead and the role of suboptimality gap is different from the first part analysis, which we will specify later. We now introduce the reward perturbation technique that can artificially create a suboptimality gap. 

\subsection{Reward Perturbation Technique}

Here we use a reward perturbation technique to create a suboptimality gap in TBSG, which is inspired by a similar argument in MDP analysis \citep{li2020breaking}. We give a proof that is different from the one in \citep{li2020breaking} and our analysis for TBSG can generalize to MDP automatically as MDP is a degenerated version of TBSG. We show that by randomly perturb the reward function, with large probability, the perturbed TBSG enjoys a suboptimality gap. First we define the perturbed TBSG.

\begin{definition} (Perturbed TBSG)
For a TBSG $\mathcal{G}=(\mathcal{S},\mathcal{A},P,r,\gamma)$, the perturbed TBSG is $\mathcal{G}_\mathrm{p}=(\mathcal{S},\mathcal{A},P,r_{\mathrm{p}},\gamma)$, where 
$$r_{\mathrm{p}}=r+\zeta$$
and $\zeta\in\mathbb{R}^{|\mathcal{S}||\mathcal{A}|}$ is a vector composed of independent random variables following uniform distribution on $[0,\xi]$.
\end{definition}

We use subscript $ \pi_\mathrm{p}$ to denote the strategy in perturbed TBSG as well as $V_\mathrm{p}$ and $Q_\mathrm{p}$. The key property of perturbed TBSG is that it enjoy a suboptimality gap of $\frac{\xi\delta(1-\gamma)}{4|\mathcal{S}|^2|\mathcal{A}|}$ with probability at least $1-\delta$.

\begin{lemma}(TBSG version of suboptimality gap lemma in \citep{li2020breaking})
For a fixed policy $\nu$, with probability at least $1-\delta$, the perturbed TBSG $\mathcal{G}_\mathrm{p}$ enjoys a suboptimality gap of $\frac{\xi\delta(1-\gamma)}{4|\mathcal{S}|^2|\mathcal{A}|}$ for the Nash equilibrium strategy and a same gap for counterstrategy of $\nu$.
\label{l7}
\end{lemma}

Our proof is substantially different from \citep{li2020breaking}, which consists of two important observation. Here we consider the case where only $r(s,a_1)$ is perturbed to $r(s,a_1)+\tau$. First, the Nash equilibrium strategy $\pi_\tau^*$ is a piecewise constant function of $\tau$, if some tie breaking rule is given. Second, we show that $Q_\tau^*(s,a_1)=k_1(\pi_\tau)\tau+b_1(\pi_\tau)$ and $Q_\tau^*(s,a_2)=k_2(\pi_\tau)\tau+b_2(\pi_\tau)$ are piecewise linear function and $k_2(\pi_\tau)\leq \gamma k_1(\pi_\tau),\forall\tau$. Intuitively, these two observation means that $Q_\tau^*(s,a_2)$ grows at most $\gamma$ times the speed of $Q_\tau^*(s,a_1)$, which further implies $|Q_\tau^*(s,a_1)-Q_\tau^*(s,a_2)|\leq w$ can only holds for a small interval of $\tau$.

\begin{lemma}
Consider a TBSG $\mathcal{G}_\tau=(\mathcal{S},\mathcal{A},P,r_\tau,\gamma)$ where $r_\tau=r+\tau\mathbf{1}_{s,a}$. Then the following facts hold.
\begin{itemize}
    \item Given a rule to select optimal action when there are multiple optimal actions, $\pi^*_\tau$ is a constant (vector) function of $\tau$.
    \item $Q^*_\tau$ is a piecewise linear (vector) function of $\tau$.
    \item $Q_\tau^*(s,a')=kQ_\tau^*(s,a)+b$, where $0\leq k\leq\gamma$ and $b$ are a function of $\pi_\tau^*$.
\end{itemize}
\label{l8}
\end{lemma}

With the above argument, we can prove that for arbitrary $s,a,a'$, if we increase $r(s,a)$, then the growth of $Q_\tau^*(s,a')$ is at most $\gamma$ times of $Q_\tau^*(s,a)$, which means $|Q_\tau^*(s,a')-Q_\tau^*(s,a)|$ only holds for a small range of $r_\tau(s,a)$. With a union bound argument, we prove the existence of suboptimality gap. The proof for counterstrategy is similar and the details are given in the appendix.

\subsection{Minimax Sample Complexity}

In this section, we show how to use the reward perturbation technique to derive the minimax sample complexity result. First, we show that the optimal strategy in perturbed empirical TBSG is contained in a finite set that has no dependence on $\widehat{P}(s,a)$.

\begin{lemma}
Set $U$ to be a set of equally spaced points in $[-\frac{1}{1-\gamma},\frac{1}{1-\gamma}]$ and $|U|=\frac{16|\mathcal{S}|^2|\mathcal{A}|}{(1-\gamma)^2\xi\delta}$. We define
$$\mathcal{M}^*=\{\widetilde{\mu}_{\mathrm{p},u}^*:u\in U\},\mathcal{M}^{\mu^*}=\{\widetilde{c}_{\mathrm{p},u}(\mu^*):u\in U\}.$$
With probability at least $1-\delta$, we have $\widehat{\mu}_\mathrm{p}^*\in\mathcal{M}^*$ and $\widehat{c}_\mathrm{p}(\mu^*)\in\mathcal{M}^{\mu^*}$.
\label{l9}
\end{lemma}

Lemma 9 means with large probability $\widehat{\mu}_\mathrm{p}^*$ and $\widehat{c}_\mathrm{p}(\mu^*)$ lie in a finite set, which is independent of $\widehat{P}(s,a)$. This independence allows the usage of Bernstein's inequality and with the union bound, we can prove the concentration of $(\widehat{P}(s,a)-P(s,a))V_\mathrm{p}^{\widehat{\mu}_\mathrm{p}^*,*}$ and $(\widehat{P}(s,a)-P(s,a))V_\mathrm{p}^{\mu^*,\widehat{c}_\mathrm{p}(\mu^*)}$. Then, with Lemma 1 and Lemma 3, we can bound $|Q_\mathrm{p}^{\widehat{\mu}_\mathrm{p}^*,*}-\widehat{Q}_\mathrm{p}^{\widehat{\mu}_\mathrm{p}^*,c(\widehat{\mu}_\mathrm{p}^*)}|$ and $|Q_\mathrm{p}^{\mu^*,\widehat{c}_\mathrm{p}(\mu^*)}-\widehat{Q}_\mathrm{p}^{\mu^*,*}|$. 

\begin{lemma}
For perturbed empirical Nash equilibrium strategy $\widehat{\pi}_\mathrm{p}^*=(\widehat{\mu}_\mathrm{p}^*,\widehat{\nu}_\mathrm{p}^*)$, we have
\begin{align*}
    |Q^{\widehat{\mu}_\mathrm{p},*}-Q^*|\leq&|Q_\mathrm{p}^{\widehat{\mu}_\mathrm{p}^*,*}-\widehat{Q}_\mathrm{p}^{\widehat{\mu}_\mathrm{p}^*,c(\widehat{\mu}_\mathrm{p}^*)}|+|Q_\mathrm{p}^{\mu^*,\widehat{c}_\mathrm{p}(\mu^*)}-\widehat{Q}_\mathrm{p}^{\mu^*,*}|+\frac{4\xi}{1-\gamma}.
\end{align*}
\label{l10}
\end{lemma}

Lemma 10 shows how to bound $|Q^{\widehat{\mu}_\mathrm{p},*}-Q^*|$ by perturbed TBSG. In the same manner, we can bound $|Q^{*,\widehat{\nu}_\mathrm{p}}-Q^*|$. Selecting an appropriate $\xi$, we can show that perturbed empirical Nash equilibrium strategy $\widehat{\pi}_\mathrm{p}^*$ is an $\epsilon$-Nash equilibrium strategy. 

\begin{theorem}
If the number of samples satisfies
$$N\geq\frac{C|\mathcal{S}||\mathcal{A}|}{(1-\gamma)^{3}\epsilon^{2}}\log(\frac{|\mathcal{S}||\mathcal{A}|}{(1-\gamma)\delta\epsilon})$$
for some constant $C$, then with probability at least $1-\delta$ and $\epsilon\in(0,(1-\gamma)^{-1}]$, we have that $\widehat{\pi}_\mathrm{p}^*$ is an $\epsilon$-approximate Nash equilibrium strategy in $\mathcal{G}$.
\label{th2}
\end{theorem}

In addition, we can derive the following improved result for problem-dependent bound by choosing $\epsilon=\frac{\Delta}{2}$ and further analysis on suboptimality gap, which is provided in the appendix.

\begin{theorem}
If $\mathcal{G}$ enjoys a suboptimality gap of $\Delta$ and the number of samples satisfies
$$N\geq\frac{C|\mathcal{S}||\mathcal{A}|}{(1-\gamma)^3\Delta^2}\log(\frac{|\mathcal{S}||\mathcal{A}|}{(1-\gamma)\delta\Delta})$$
for some constant $C$ and $\Delta\in(0,(1-\gamma)^{-1}]$, then with probability at least $1-\delta$, we have $\widehat{\pi}_\mathrm{p}^*=\pi^*$, which means the empirical Nash equilibrium strategy we obtained is exactly the Nash equilibrium strategy in the true TBSG.
\label{th3}
\end{theorem}

\section{Related Literature}

\paragraph{TBSG} TBSG has been widely studied since \citep{shapley1953stochastic}. For a detailed introduction of stochastic game, readers can refer to \citep{neyman2003stochastic}. In the old days, people focus on dynamic programming type algorithms to solve TBSG. Strategy iteration, as the counterpart of value iteration in MDP and parallelized simplex method, is proved to be a strong polynomial time algorithm \citep{hansen2013strategy,jia2020towards}. Reinforcement learning approach has been studied recently for TBSG to relieve the high computational cost of dynamic programming. Several works are proposed for the generative model setting, \citep{sidford2020solving,jia2019feature,zhang2020model}. \citep{sidford2020solving} first gives a sample efficient algorithm for tabular TBSG, while their result achieves minimax sample complexity $\widetilde{O}(|\mathcal{S}||\mathcal{A}|(1-\gamma)^{-3}\epsilon^{-2})$ only for $\epsilon\in(0,1]$. \citep{jia2019feature} adapts the MDP algorithm in \citep{sidford2018variance} for feature-based TBSG, but leaves a gap of $\frac{1}{1-\gamma}$ between optimal sample complexity. \citep{cui2020plugin} uses a similar algorithm as ours in feature-based TBSG, while their result only holds for $\epsilon$-Nash equilibrium value, which results in a $\frac{1}{(1-\gamma)^2}$ gap in finding $\epsilon$-Nash equilibrium strategy. \citep{zhang2020model} considers simultaneous stochastic game, and their approach consists of solving a regularized simultaneous stochastic game, which is computationally costly. For the online sampling setting, a recent work \citep{bai2020near} uses an upper confidence bound algorithms that can find an approximate Nash equilibrium strategy in $\widetilde{O}(|\mathcal{S}||\mathcal{A}||\mathcal{B}|)$ steps.

\paragraph{Generative Model} Generative model is a sampling oracle setting in MDP, which has been shown to simplify the exploration and exploitation tradeoff. This concept is formalized in \citep{kakade2003sample} and a $\widetilde{O}(|\mathcal{S}||\mathcal{A}|\mathrm{poly}((1-\gamma)^{-1})\epsilon^{-2})$ sample complexity has been proved there. \citep{azar2013minimax} proves the minimax sample complexity $\widetilde{O}(|\mathcal{S}||\mathcal{A}|(1-\gamma)^{-3}\epsilon^{-2})$. However, the upper bound there is only for $\epsilon\in(1-\gamma)^{-1/2}|\mathcal{S}|^{-1/2}$. Many works have devoted to improve the dependence on $\epsilon$. Recently, \citep{sidford2018variance} gives a minimax model-free algorithm for $\epsilon\in(0,1]$ and \citep{agarwal2019optimality} gives a minimax model-based algorithm for $\epsilon\in(0,(1-\gamma)^{-1/2}]$. Finally, \citep{li2020breaking} uses a perturbed MDP technique to prove minimax sample complexity with full range of $\epsilon$.

\paragraph{Suboptimality Gap} Suboptimality gap originated in bandit theory. Multi-armed bandits and linear bandits enjoy a logarithmic gap-dependent regret and a square root gap-independent regret \citep{auer2002finite,abbasi2011improved}. MDP with suboptimality gap have been studied in \citep{auer2009near} and recently $\widetilde{O}(|\mathcal{S}||\mathcal{A}|\mathrm{poly}(H)\log(T))$ regret has been proved for both model-based and model free algorithms \citep{simchowitz2019non,yang2020q}. \citep{du2020agnostic} utilized the suboptimality gap in general function approximation setting and proved that optimal policy can be found in $\widetilde{O}(\mathrm{dim}_\mathrm{E})$ trajectories in deterministic MDP. Most of the gap-dependent analysis in MDP focus on online RL and to the best of our knowledge, we are the first to study this notion in TBSG with a generative model.

\section{Conclusion}
In this work, we completely solve the sample complexity problem of TBSG with generative model oracle. We prove that the simplest model-based algorithm, plug-in solver approach, is minimax sample optimal for full range of $\epsilon$ by using absorbing TBSG and reward perturbation techniques. Our proof is based on suboptimality gap, a notion originated from bandit theory and receives great attention in RL. We believe that our work can shed some light on suboptimality gap and TBSG.

\clearpage

\bibliography{2-TBSG}

\begin{thebibliography}{27}
\providecommand{\natexlab}[1]{#1}
\providecommand{\url}[1]{\texttt{#1}}
\expandafter\ifx\csname urlstyle\endcsname\relax
  \providecommand{\doi}[1]{doi: #1}\else
  \providecommand{\doi}{doi: \begingroup \urlstyle{rm}\Url}\fi

\bibitem[Abbasi-Yadkori et~al.(2011)Abbasi-Yadkori, P{\'a}l, and
  Szepesv{\'a}ri]{abbasi2011improved}
Yasin Abbasi-Yadkori, D{\'a}vid P{\'a}l, and Csaba Szepesv{\'a}ri.
\newblock Improved algorithms for linear stochastic bandits.
\newblock In \emph{Advances in Neural Information Processing Systems}, pages
  2312--2320, 2011.

\bibitem[Agarwal et~al.(2019)Agarwal, Kakade, and Yang]{agarwal2019optimality}
Alekh Agarwal, Sham Kakade, and Lin~F Yang.
\newblock On the optimality of sparse model-based planning for markov decision
  processes.
\newblock \emph{arXiv preprint arXiv:1906.03804}, 2019.

\bibitem[Auer et~al.(2002)Auer, Cesa-Bianchi, and Fischer]{auer2002finite}
Peter Auer, Nicolo Cesa-Bianchi, and Paul Fischer.
\newblock Finite-time analysis of the multiarmed bandit problem.
\newblock \emph{Machine learning}, 47\penalty0 (2-3):\penalty0 235--256, 2002.

\bibitem[Auer et~al.(2009)Auer, Jaksch, and Ortner]{auer2009near}
Peter Auer, Thomas Jaksch, and Ronald Ortner.
\newblock Near-optimal regret bounds for reinforcement learning.
\newblock In \emph{Advances in neural information processing systems}, pages
  89--96, 2009.

\bibitem[Azar et~al.(2013)Azar, Munos, and Kappen]{azar2013minimax}
Mohammad~Gheshlaghi Azar, R{\'e}mi Munos, and Hilbert~J Kappen.
\newblock Minimax pac bounds on the sample complexity of reinforcement learning
  with a generative model.
\newblock \emph{Machine learning}, 91\penalty0 (3):\penalty0 325--349, 2013.

\bibitem[Bai et~al.(2020)Bai, Jin, and Yu]{bai2020near}
Yu~Bai, Chi Jin, and Tiancheng Yu.
\newblock Near-optimal reinforcement learning with self-play.
\newblock \emph{arXiv preprint arXiv:2006.12007}, 2020.

\bibitem[Cui and Yang(2020)]{cui2020plugin}
Qiwen Cui and Lin~F. Yang.
\newblock Is plug-in solver sample-efficient for feature-based reinforcement
  learning?, 2020.

\bibitem[Du et~al.(2020)Du, Lee, Mahajan, and Wang]{du2020agnostic}
Simon~S Du, Jason~D Lee, Gaurav Mahajan, and Ruosong Wang.
\newblock Agnostic q-learning with function approximation in deterministic
  systems: Tight bounds on approximation error and sample complexity.
\newblock \emph{arXiv preprint arXiv:2002.07125}, 2020.

\bibitem[Hansen et~al.(2013)Hansen, Miltersen, and Zwick]{hansen2013strategy}
Thomas~Dueholm Hansen, Peter~Bro Miltersen, and Uri Zwick.
\newblock Strategy iteration is strongly polynomial for 2-player turn-based
  stochastic games with a constant discount factor.
\newblock \emph{Journal of the ACM (JACM)}, 60\penalty0 (1):\penalty0 1--16,
  2013.

\bibitem[Jia et~al.(2019)Jia, Yang, and Wang]{jia2019feature}
Zeyu Jia, Lin~F Yang, and Mengdi Wang.
\newblock Feature-based q-learning for two-player stochastic games.
\newblock \emph{arXiv preprint arXiv:1906.00423}, 2019.

\bibitem[Jia et~al.(2020)Jia, Wen, and Ye]{jia2020towards}
Zeyu Jia, Zaiwen Wen, and Yinyu Ye.
\newblock Towards solving 2-tbsg efficiently.
\newblock \emph{Optimization Methods and Software}, 35\penalty0 (4):\penalty0
  706--721, 2020.

\bibitem[Kaiser et~al.(2019)Kaiser, Babaeizadeh, Milos, Osinski, Campbell,
  Czechowski, Erhan, Finn, Kozakowski, Levine, et~al.]{kaiser2019model}
Lukasz Kaiser, Mohammad Babaeizadeh, Piotr Milos, Blazej Osinski, Roy~H
  Campbell, Konrad Czechowski, Dumitru Erhan, Chelsea Finn, Piotr Kozakowski,
  Sergey Levine, et~al.
\newblock Model-based reinforcement learning for atari.
\newblock \emph{arXiv preprint arXiv:1903.00374}, 2019.

\bibitem[Kakade et~al.(2003)]{kakade2003sample}
Sham~Machandranath Kakade et~al.
\newblock \emph{On the sample complexity of reinforcement learning}.
\newblock PhD thesis, University of London London, England, 2003.

\bibitem[Kearns and Singh(1999)]{kearns1999finite}
Michael~J Kearns and Satinder~P Singh.
\newblock Finite-sample convergence rates for q-learning and indirect
  algorithms.
\newblock In \emph{Advances in neural information processing systems}, pages
  996--1002, 1999.

\bibitem[Li et~al.(2020)Li, Wei, Chi, Gu, and Chen]{li2020breaking}
Gen Li, Yuting Wei, Yuejie Chi, Yuantao Gu, and Yuxin Chen.
\newblock Breaking the sample size barrier in model-based reinforcement
  learning with a generative model.
\newblock \emph{arXiv preprint arXiv:2005.12900}, 2020.

\bibitem[Neyman et~al.(2003)Neyman, Sorin, and Sorin]{neyman2003stochastic}
Abraham Neyman, Sylvain Sorin, and S~Sorin.
\newblock \emph{Stochastic games and applications}, volume 570.
\newblock Springer Science \& Business Media, 2003.

\bibitem[Shapley(1953)]{shapley1953stochastic}
Lloyd~S Shapley.
\newblock Stochastic games.
\newblock \emph{Proceedings of the national academy of sciences}, 39\penalty0
  (10):\penalty0 1095--1100, 1953.

\bibitem[Sidford et~al.(2018)Sidford, Wang, Wu, and Ye]{sidford2018variance}
Aaron Sidford, Mengdi Wang, Xian Wu, and Yinyu Ye.
\newblock Variance reduced value iteration and faster algorithms for solving
  markov decision processes.
\newblock In \emph{Proceedings of the Twenty-Ninth Annual ACM-SIAM Symposium on
  Discrete Algorithms}, pages 770--787. SIAM, 2018.

\bibitem[Sidford et~al.(2020)Sidford, Wang, Yang, and Ye]{sidford2020solving}
Aaron Sidford, Mengdi Wang, Lin Yang, and Yinyu Ye.
\newblock Solving discounted stochastic two-player games with near-optimal time
  and sample complexity.
\newblock In \emph{International Conference on Artificial Intelligence and
  Statistics}, pages 2992--3002. PMLR, 2020.

\bibitem[Silver et~al.(2017)Silver, Schrittwieser, Simonyan, Antonoglou, Huang,
  Guez, Hubert, Baker, Lai, Bolton, et~al.]{silver2017mastering}
David Silver, Julian Schrittwieser, Karen Simonyan, Ioannis Antonoglou, Aja
  Huang, Arthur Guez, Thomas Hubert, Lucas Baker, Matthew Lai, Adrian Bolton,
  et~al.
\newblock Mastering the game of go without human knowledge.
\newblock \emph{nature}, 550\penalty0 (7676):\penalty0 354--359, 2017.

\bibitem[Simchowitz and Jamieson(2019)]{simchowitz2019non}
Max Simchowitz and Kevin~G Jamieson.
\newblock Non-asymptotic gap-dependent regret bounds for tabular mdps.
\newblock In \emph{Advances in Neural Information Processing Systems}, pages
  1153--1162, 2019.

\bibitem[Sutton and Barto(2018)]{sutton2018reinforcement}
Richard~S Sutton and Andrew~G Barto.
\newblock \emph{Reinforcement learning: An introduction}.
\newblock MIT press, 2018.

\bibitem[Tesauro(1995)]{tesauro1995temporal}
Gerald Tesauro.
\newblock Temporal difference learning and td-gammon.
\newblock \emph{Communications of the ACM}, 38\penalty0 (3):\penalty0 58--68,
  1995.

\bibitem[Wang et~al.(2019)Wang, Bao, Clavera, Hoang, Wen, Langlois, Zhang,
  Zhang, Abbeel, and Ba]{wang2019benchmarking}
Tingwu Wang, Xuchan Bao, Ignasi Clavera, Jerrick Hoang, Yeming Wen, Eric
  Langlois, Shunshi Zhang, Guodong Zhang, Pieter Abbeel, and Jimmy Ba.
\newblock Benchmarking model-based reinforcement learning.
\newblock \emph{arXiv preprint arXiv:1907.02057}, 2019.

\bibitem[Yang et~al.(2020)Yang, Yang, and Du]{yang2020q}
Kunhe Yang, Lin~F Yang, and Simon~S Du.
\newblock $ q $-learning with logarithmic regret.
\newblock \emph{arXiv preprint arXiv:2006.09118}, 2020.

\bibitem[Ye et~al.(2020)Ye, Liu, Sun, Shi, Zhao, Wu, Yu, Yang, Wu, Guo,
  et~al.]{ye2020mastering}
Deheng Ye, Zhao Liu, Mingfei Sun, Bei Shi, Peilin Zhao, Hao Wu, Hongsheng Yu,
  Shaojie Yang, Xipeng Wu, Qingwei Guo, et~al.
\newblock Mastering complex control in moba games with deep reinforcement
  learning.
\newblock In \emph{AAAI}, pages 6672--6679, 2020.

\bibitem[Zhang et~al.(2020)Zhang, Kakade, Ba{\c{s}}ar, and
  Yang]{zhang2020model}
Kaiqing Zhang, Sham~M Kakade, Tamer Ba{\c{s}}ar, and Lin~F Yang.
\newblock Model-based multi-agent rl in zero-sum markov games with near-optimal
  sample complexity.
\newblock \emph{arXiv preprint arXiv:2007.07461}, 2020.

\end{thebibliography}
\bibliographystyle{plainnat}

\appendix

\section{Technical Lemmas from MDP}

\paragraph{Additional Notations} We use $\Pi^\pi:\mathbb{R}^{|\mathcal{S}||\mathcal{A}|}\rightarrow\mathbb{R}^{|\mathcal{S}|}$ to denote the projection operator with respect to policy $\pi$, which means if $V=\Pi^\pi Q$, then $V(s)=Q(s,\pi(s)),\forall s\in \mathcal{S}$. We use  $\widehat{\mathcal{G}}$ to denote the empirical TBSG, $\widetilde{\mathcal{G}}$ to denote the absorbing TBSG on empirical TBSG and $\mathcal{G}_\mathrm{p}$ to denote the perturbed TBSG. We use $\mathbf{1}$ to denote a vector with all entries to be $1$ and $\mathbf{1}_{s,a}$ to denote a zero vector with only $(s,a)$ entry to be $1$.

\begin{proof}[Proof of Lemma \ref{l1}]

By Bellman equation, we have $Q^\pi=(I-\gamma P^\pi)^{-1}r$ and $\widehat{Q}^\pi=(I-\gamma \widehat{P}^\pi)^{-1}r$. 
For any policy $\pi$, we have
\begin{align*}
	Q^\pi-\widehat{Q}^\pi&=(I-\gamma P^\pi)^{-1}r-(I-\gamma\widehat{P}^\pi)^{-1}r\\
	&=(I-\gamma P^\pi)^{-1}((I-\gamma\widehat{P}^\pi)-(I-\gamma P^\pi))\widehat{Q}^\pi\\
	&=\gamma(I-\gamma P^\pi)^{-1}(P^\pi-\widehat{P}^\pi)\widehat{Q}^\pi\\
	&=\gamma(I-\gamma P^\pi)^{-1}(P-\widehat{P})\widehat{V}^\pi.
\end{align*}
Similarly, we have $Q^\pi-\widehat{Q}^\pi=\gamma(I-\gamma \widehat{P}^\pi)^{-1}(\widehat{P}-P)V^\pi.$

\end{proof}

\begin{lemma}
(Lemma 3 in \citep{sidford2020solving}) For any policy $\pi$, we have
$$\left|(I-\gamma P^\pi)^{-1}\sqrt{\mathrm{Var}_\mathrm{P}(V^\pi)}\right|\leq\sqrt{\frac{2}{(1-\gamma)^3}}.$$
\label{l11}
\end{lemma}

\begin{proof}[Proof of Lemma \ref{l2}]
We have
\begin{align*}
    \left|(I-\gamma P^\pi)^{-1}\sqrt{\mathrm{Var}_\mathrm{P}(\widehat{V}^\pi)}\right|&\leq\left|(I-\gamma P^\pi)^{-1}\sqrt{\mathrm{Var}_\mathrm{P}(V^\pi)}\right|+\left|(I-\gamma P^\pi)^{-1}\sqrt{\mathrm{Var}_\mathrm{P}(\widehat{V}^\pi-V^\pi)}\right|\\
    &\leq \sqrt{\frac{2}{(1-\gamma)^3}}+\frac{|\widehat{V}^\pi-V^\pi|}{1-\gamma}\\
    &\leq \sqrt{\frac{2}{(1-\gamma)^3}}+\frac{|\widehat{Q}^\pi-Q^\pi|}{1-\gamma},\\
\end{align*}
where the first inequality is the triangle inequality, the second one is from Lemma \ref{l11} and the fact that $|(I-\gamma P^\pi)^{-1}V|\leq\frac{|V|}{1-\gamma}$, and the last one is because $V^\pi$ is a subset of $Q^\pi$.

\end{proof}

\begin{lemma}
(Lemma 8 in \citep{li2020breaking}) Let $Q$ be a vector obeying $Q=(I-\gamma P^\pi)^{-1}r$ for some vector $r>0$ and $V=\Pi^\pi Q$, then we have
$$\left|(I-\gamma P^\pi)^{-1}\sqrt{\mathrm{Var}_\mathrm{P}(V)}\right|\leq\frac{4}{\gamma\sqrt{1-\gamma}}|Q|.$$
\label{l12}
\end{lemma}

\begin{proof}[Proof of Lemma \ref{l3}]

We use a Taylor expansion type analysis. For simplicity, we define 
$$r^{(n)}=\sqrt{\mathrm{Var}_\mathrm{P}(V^{(n-1)})}, Q^{(n)}=(I-\gamma P^\pi)^{-1}r^{(n)}, V^{(n)}=\Pi^{\pi} Q^{(n)}, n=1,2,\cdots$$
and $V^{(0)}=V^\pi$. Then with large probability we have
\begin{align*}
    |(I-\gamma \widehat{P}^\pi)^{-1}\sqrt{\mathrm{Var}_\mathrm{P}(V^\pi)}|&=|(I-\gamma \widehat{P}^\pi)^{-1}r^{(1)}|\\
    &=|(I-\gamma P^\pi)^{-1}r^{(1)}+(I-\gamma \widehat{P}^\pi)^{-1}(\gamma \widehat{P}^\pi-\gamma P^\pi)(I-\gamma P^\pi)^{-1}r^{(1)}|\\
    &=|(I-\gamma P^\pi)^{-1}r^{(1)}+(I-\gamma \widehat{P}^\pi)^{-1}(\gamma \widehat{P}^\pi-\gamma P^\pi)Q^{(1)}|\\
    &\lesssim |(I-\gamma P^\pi)^{-1}r^{(1)}|+\frac{\gamma}{\sqrt{N}}|(I-\gamma \widehat{P}^\pi)^{-1}\sqrt{\mathrm{Var}_\mathrm{P}(V^{(1)})}|\\
    &=|(I-\gamma P^\pi)^{-1}r^{(1)}|+\frac{\gamma}{\sqrt{N}}|(I-\gamma \widehat{P}^\pi)^{-1}r^{(2)}|\\
    &\lesssim \cdots\\
    &\lesssim \sum_{i=1}^n (\frac{\gamma}{\sqrt{N}})^{i-1} |(I-\gamma P^\pi)^{-1}r^{(i)}|+(\frac{\gamma}{\sqrt{N}})^{n} |(I-\gamma \widehat{P}^\pi)^{-1}r^{(n+1)}|
\end{align*}

In addition, by Lemma \ref{l12}, we have
\begin{align*}
    \left|(I-\gamma P^\pi)^{-1}r^{(i)}\right|&=\left|(I-\gamma P^\pi)^{-1}\sqrt{\mathrm{Var}_\mathrm{P}(V^{(i-1)})}\right|\\
    &\leq\frac{4}{\gamma\sqrt{1-\gamma}}|Q^{i-1}|\\
    &\leq\frac{4}{\gamma\sqrt{1-\gamma}}\left|(I-\gamma \widehat{P}^\pi)^{-1}r^{(i-1)}\right|.
\end{align*}

By induction, we have $\left|(I-\gamma P^\pi)^{-1}r^{(i)}\right|\leq\left(\frac{4}{\gamma\sqrt{1-\gamma}}\right)^{i}\left|(I-\gamma \widehat{P}^\pi)^{-1}r^{(0)}\right|\leq\left(\frac{4}{\gamma\sqrt{1-\gamma}}\right)^{i}\frac{1}{1-\gamma}$. If we set $N\gtrsim\frac{64}{1-\gamma}$ and $n\gtrsim\log(\frac{1}{1-\gamma})$, we have
\begin{align*}
    \left|(I-\gamma \widehat{P}^\pi)^{-1}\sqrt{\mathrm{Var}_\mathrm{P}(V^\pi)}\right|\leq \frac{4}{\gamma\sqrt{(1-\gamma)^3}}\sum_{i=1}^{n}\left(\frac{4}{\sqrt{N(1-\gamma)}}\right)^{i-1}+\frac{1}{(1-\gamma)^2}\left(\frac{4}{\sqrt{N(1-\gamma)}}\right)^{n}
    \leq \frac{16}{\sqrt{(1-\gamma)^3}}
\end{align*}

For a more detailed proof on lower order term, one can refer to \citep{li2020breaking}.
\end{proof}

\section{Problem-dependent Upper Bound}

\begin{proof}[Proof of Lemma \ref{l4}]
First, we show that choosing $u=u^*$ in absorbing TBSG can recover $\widehat{Q}^*$.
\begin{align*}
    \widetilde{Q}_{u^*}^{\widehat{\pi}^*}&=(I-\gamma \widetilde{P}^{\widehat{\pi}^*})^{-1}(r+(u^*-r(s,a))\mathbf{1}_{s,a})\\
    &=(I-\gamma \widetilde{P}^{\widehat{\pi}^*})^{-1}((I-\gamma\widehat{P}^{\widehat{\pi}^*})\widehat{Q}^*+\gamma(\widehat{P}\widehat{V}^*-\widetilde{P} \widehat{V}^*))\\
    &=\widehat{Q}^*.
\end{align*}

By the property of Nash equilibrium strategy, we have $\widehat{Q}^*=\widetilde{Q}_{u^*}^{\widehat{\pi}^*}=\widetilde{Q}_{u^*}^*$. Specifically, $\widehat{\pi}^*$ select the optimal value in $\widehat{Q}^*$, which means it select the optimal value in $\widetilde{Q}_{u^*}^{\widehat{\pi}^*}$, and this is equivalent to $\widehat{\pi}^*=\widetilde{\pi}_{u^*}^*$.

Second, we show that the optimal Q-value in absorbing TBSG is $\frac{1}{1-\gamma}$-lipschitz to $u$. This is proved by bounding the distance between an upper bound and an lower bound.
\begin{align*}
    \left|\widetilde{Q}_{u}^{\widetilde{\mu}_{u}^*,\widetilde{\nu}_{u'}^*}-\widetilde{Q}_{u'}^{\widetilde{\mu}_{u}^*,\widetilde{\nu}_{u'}^*}\right|&=\left|(I-\widetilde{P}^{\widetilde{\mu}_{u}^*,\widetilde{\nu}_{u'}^*})^{-1}(r_{u}-r_{u'})\right|\\
    &=\left|(u-u')(I-\widetilde{P}^{\widetilde{\mu}_{u}^*,\widetilde{\nu}_{u'}^*})^{-1}\mathbf{1}_{s,a}\right|\\
    &\leq\frac{|u-u'|}{1-\gamma}.
\end{align*}
Similarly, we have $\left|\widetilde{Q}_{u}^{\widetilde{\mu}_{u'}^*,\widetilde{\nu}_{u}^*}-\widetilde{Q}_{u'}^{\widetilde{\mu}_{u'}^*,\widetilde{\nu}_{u}^*}\right|\leq\frac{|u-u'|}{1-\gamma}$. By the definition of $Q^*$, we have $\widetilde{Q}_{u}^{\widetilde{\mu}_{u'}^*,\widetilde{\nu}_{u}^*}\leq\widetilde{Q}_u^*\leq\widetilde{Q}_{u}^{\widetilde{\mu}_{u}^*,\widetilde{\nu}_{u'}^*}$ and $\widetilde{Q}_{u'}^{\widetilde{\mu}_{u}^*,\widetilde{\nu}_{u'}^*}\leq\widetilde{Q}_{u'}^*\leq\widetilde{Q}_{u'}^{\widetilde{\mu}_{u'}^*,\widetilde{\nu}_{u}^*}$. Thus we have $|\widetilde{Q}_{u}^*-\widetilde{Q}_{u'}^*|\leq\frac{|u-u'|}{1-\gamma}$. The proof for counterstrategy is the same.

\end{proof}

\begin{lemma}
For any fixed value vector $V$, with probability larger than $1-\delta$, we have
$$|(P(s,a)-\widehat{P}(s,a))V|\leq\sqrt{2\log(4/\delta)}\sqrt{\frac{Var_{s,a}(V)}{N}}+\frac{2\log(4/\delta)}{3(1-\gamma)N},$$
where $N$ is the sample size from distribution $P(\cdot|s,a)$.
\label{l13}
\end{lemma}

\begin{proof}
The proof is a direct application of Bernstein's inequality.
\end{proof}

\begin{lemma}
For $C=2\log(\frac{32}{(1-\gamma)^2\epsilon\delta})$, with probability larger than $1-\delta$, we have
$$\left|(P(s,a)-\widehat{P}(s,a))\widehat{V}^{\mu^*,*}\right|\leq \sqrt{\frac{C \mathrm{Var}_{s,a}(\widehat{V}^{\mu^*,*})}{N}}+\frac{C}{3(1-\gamma)N}+\left(\sqrt{\frac{C}{N}}+1\right)\frac{\epsilon(1-\gamma)}{4},$$
where $N$ is the sample size from distribution $P(\cdot|s,a)$.
\label{l14}
\end{lemma}

\begin{proof}
We define a fixed set $U$ to be evenly spaced points in $[-\frac{1}{1-\gamma},\frac{1}{1-\gamma}]$ such that $|U|=\frac{8}{(1-\gamma)^2\epsilon}$. Combining Lemma \label{l12} with union bound, with probability larger than $1-\delta$, we have for all $u\in U$,
$$\left|(P(s,a)-\widehat{P}(s,a))\widetilde{V}^{\mu^*,*}_u\right|\leq\sqrt{2\log(4|U|/\delta)}\sqrt{\frac{Var_{s,a}(\widetilde{V}^{\mu^*,*}_u)}{N}}+\frac{2\log(4|U|/\delta)}{3(1-\gamma)N}.$$

For each $u\in U$, we have
\begin{align*}
    \left|(P(s,a)-\widehat{P}(s,a))\widehat{V}^{\mu^*,*}\right|&=\left|(P(s,a)-\widehat{P}(s,a))\widetilde{V}_{u^{\mu^*}}^{\mu^*,*}\right|\\
    &\leq\left|(P(s,a)-\widehat{P}(s,a))\widetilde{V}_{u}^{\mu^*,*}\right|+\frac{|u-u^{\mu^*}|}{1-\gamma}\\
    &\leq\sqrt{2\log(4|U|/\delta)}\sqrt{\frac{Var_{s,a}(\widetilde{V}^{\mu^*,*}_u)}{N}}+\frac{2\log(4|U|/\delta)}{3(1-\gamma)N}+\frac{|u-u^{\mu^*}|}{1-\gamma}\\
    &\leq\sqrt{2\log(4|U|/\delta)}\sqrt{\frac{Var_{s,a}(\widetilde{V}^{\mu^*,*}_{u^{\mu^*}})}{N}}+\frac{2\log(4|U|/\delta)}{3(1-\gamma)N}+\left(1+\sqrt{\frac{2\log(4|U|/\delta)}{N}}\right)\frac{|u-u^{\mu^*}|}{1-\gamma}\\
\end{align*}

As there exist a $u\in U$ such that $|u-u^{\mu^*}|\leq\frac{\epsilon(1-\gamma)^2}{4}$, we can prove the argument.
\end{proof}

\begin{proof}[Proof of Lemma \ref{l5}]
By the property of counterstrategy, we have
\begin{align*}
    Q^*-\widehat{Q}^*&=Q^*-Q^{\widehat{c}(\nu^*),\nu^*}+Q^{\widehat{c}(\nu^*),\nu^*}-\widehat{Q}^{*,\nu^*}+\widehat{Q}^{*,\nu^*}-\widehat{Q}^*\\
    &\geq Q^{\widehat{c}(\nu^*),\nu^*}-\widehat{Q}^{*,\nu^*}.
\end{align*}
Similarly, we have $Q^*-\widehat{Q}^*\leq Q^{\mu^*,\widehat{c}(\mu^*)}-\widehat{Q}^{\mu^*,*}$. Together we can prove Lemma \ref{l5}.
\end{proof}

\begin{proof}[Proof of Lemma \ref{l6}]
By the definition of suboptimality gap, for all $s\in \mathcal{S}_{\mathrm{max}},a\in\mathcal{A}$ we have
\begin{align*}
    \widehat{Q}^*(s,\mu^*(s))-\widehat{Q}^*(s,a)&=\widehat{Q}^*(s,\mu^*(s))-Q^*(s,\mu^*(s))+Q^*(s,\mu^*(s))-Q^*(s,a)+Q^*(s,a)-\widehat{Q}^*(s,a)\\
    &\geq-\frac{\Delta}{2}+\Delta-\frac{\Delta}{2}\\
    &\geq0.
\end{align*}

Similar, for all $s\in \mathcal{S}_{\mathrm{min}},a\in\mathcal{A}$, we have
$$\widehat{Q}^*(s,\nu^*(s))-\widehat{Q}^*(s,a)< 0.$$
Note that the empirical Nash equilibrium strategy $\widehat{\pi}^*$ is the only policy that can satisfy the above two conditions, which means $\pi^*=\widehat{\pi}^*$.
\end{proof}

\begin{proof}[Proof of Theorem \ref{th1}]

We set $N=\frac{32|\mathcal{S}||\mathcal{A}|}{(1-\gamma)^3\epsilon^2}\log(\frac{32}{(1-\gamma)^2\epsilon\delta})$. With probability larger than $1-\delta$, We have
\begin{align*}
    &\left|Q^{\mu^*,\widehat{c}(\mu^*)}-\widehat{Q}^{\mu^*,*}\right|\\=&\left|\gamma(I-\gamma P^{\mu^*,\widehat{c}(\mu^*)})^{-1}(P-\widehat{P})\widehat{V}^{\mu^*,\widehat{c}(\mu^*)}\right|\\
    \leq& \left|\gamma(I-\gamma P^{\mu^*,\widehat{c}(\mu^*)})^{-1}\left\{\sqrt{\frac{C\mathrm{Var}_{\mathrm{P}}(\widehat{V}^{\mu^*,*})}{N/|\mathcal{S}||\mathcal{A}|}}+\left[\frac{C}{3(1-\gamma)N/|\mathcal{S}||\mathcal{A}|}+\left(\sqrt{\frac{C}{N/|\mathcal{S}||\mathcal{A}|}}+1\right)\frac{\epsilon(1-\gamma)}{4}\right]\mathbf{1}\right\}\right|\\
    \leq& \sqrt{\frac{2C}{(1-\gamma)^3N/|\mathcal{S}||\mathcal{A}|}}+\sqrt{\frac{C}{N/|\mathcal{S}||\mathcal{A}|}}\frac{\left|Q^{\mu^*,\widehat{c}(\mu^*)}-\widehat{Q}^{\mu^*,*}\right|}{1-\gamma}+\frac{C}{3(1-\gamma)^2N/|\mathcal{S}||\mathcal{A}|}+\left(\sqrt{\frac{C}{N/|\mathcal{S}||\mathcal{A}|}}+1\right)\frac{\epsilon}{4}\\
\end{align*}
where the first equality is due to Lemma \ref{l1}, the first inequality is due to Lemma \ref{l14}, the second inequality is due to Lemma \ref{l2}.

Thus for $\epsilon\leq\frac{1}{\sqrt{1-\gamma}}$ and $N\geq\frac{32|\mathcal{S}||\mathcal{A}|}{(1-\gamma)^3\epsilon^2}\log(\frac{32}{(1-\gamma)^2\epsilon\delta})$, we have
$$\left|Q^{\mu^*,\widehat{c}(\mu^*)}-\widehat{Q}^{\mu^*,*}\right|\leq\epsilon.$$

Similarly, we have $|Q^{\widehat{c}(\nu^*),\nu^*}-\widehat{Q}^{*,\nu^*}|\leq\epsilon$. Finally, we set $\epsilon=\frac{\Delta}{2}$. Then by Lemma \ref{l5} and Lemma \ref{l6}, we can conclude that with probability $1-\delta$ and $N=\frac{128|\mathcal{S}||\mathcal{A}|}{(1-\gamma)^3\Delta^2}\log(\frac{128}{(1-\gamma)^2\Delta\delta})$, the empirical Nash equilibrium strategy $\widehat{\pi}^*$ is exactly the true Nash equilibrium strategy $\pi^*$.

\end{proof}

\section{Problem-independent Upper Bound}

\begin{proof}[Proof of Lemma \ref{l8}]
As $Q_\tau^*=\max_\pi Q_\tau^\pi=\max_\pi (I-\gamma P^\pi)^{-1}r_\tau$ is a continuous function of $\tau$, with a tie breaking rule to select optimal action if there are multiple optimal actions, the optimal policy $\pi_\tau^*$ is a piecewise constant function.

For a given policy $\pi$, the corresponding Q-function satisfies $Q_\tau^\pi=(I-\gamma P^\pi)^{-1}(r+\tau\mathbf{1}_{s,a})$, which means $Q_\tau^\pi$ is a linear function of $\tau$. As $\pi_\tau^*$ is a piecewise constant function, we have that $Q_\tau^*$ is a piecewise linear function.

For the last argument, we factorize $Q_\tau^*(s,a')$ according to the hitting time of state $s$.

\begin{align*}
    Q_\tau^*(s,a')&=r(s,a')+\gamma P(s,a')V_\tau^*\\
    &=\gamma p_1V_\tau(s)+\gamma \sum_{s'\neq s} P(s'|s,a')V_\tau^*(s')+r(s,a')\\
    &=\gamma p_1V_\tau(s)+\gamma^2 p_2V_\tau(s)+\gamma^2\sum_{s'',s'\neq s}P(s'|s,a')P(s''|s',\pi^*(s'))V_\tau^*(s'')+r(s,a')+\gamma\sum_{s'\neq s}P(s'|s,a)r(s',\pi^*(s'))\\
    &=\cdots\\
    &=\sum_{n=1}^\infty \gamma^n p_n V^*_\tau(s)+b(\pi_\tau^*),
\end{align*}
where $p_n$ is the probability of first visiting state $s$ in step $n$ under optimal policy and $b(\pi_\tau^*)$ is a function of $\pi_\tau^*$. For $a=\pi_\tau^*(s)$, we define $k=\sum_{n=1}^\infty \gamma^n p_n\leq\gamma$, then $Q_\tau^*(s,a')=kQ^*_\tau(s,a)+b(\pi_\tau^*)$. When $\pi_\tau^*(s)\neq a$, $Q_\tau^*(s,a')=r(s,a')+\gamma P(s,a')(I-\gamma P_{\pi_\tau^*})^{-1}r_{\pi_\tau^*}$ is a function of $\pi_\tau^*$, which means we have $Q_\tau^*(s,a')=0Q_\tau^*(s,a)+b(\pi_\tau^*)$.
\end{proof}

\begin{proof}[Proof of Lemma \ref{l7}]
We prove that for any state $s$ and actions $a,a'$, with large probability, a gap between $Q_{\mathrm{p}}^*(s,a)$ and $Q_{\mathrm{p}}^*(s,a')$ exist. Then with a union bound, Lemma \ref{l7} holds with large probability.

Now we fix all rewards except $r_\mathrm{p}(s,a)=r(s,a)+\tau$. We define $\mathcal{I}_w=\{\tau||Q_\tau^*(s,a)-Q_\tau^*(s,a')|\leq w\}$. First, we have $Q_{\tau_1}^*(s,a)-Q_{\tau_2}^*(s,a)=\tau_1-\tau_2+P(s,a)(V_{\tau_1}^*-V_{\tau_2}^*)\geq\tau_1-\tau_2,\forall \tau_1\geq\tau_2$. In addition, Lemma \ref{l8} implies that the growth rate of $Q_{\tau}^*(s,a')$ is at most $\gamma$ times the rate of $Q_{\tau}(s,a)$. Thus we can conclude that the length of $\mathcal{I}_w$ is at most $\frac{2w}{1-\gamma}$. We set $w=\frac{\delta\xi(1-\gamma)}{2|\mathcal{S}||\mathcal{A}|^2}$, then gaps between all $Q_{\mathrm{p}}^*(s,a)$ and $Q_{\mathrm{p}}^*(s,a')$ exist, which implies Lemma \ref{l7}. The proof for counterstrategy is the same.
\end{proof}

\begin{proof}[Proof of Lemma \ref{l9}]
Select $u\in U$ such that $u-u^*\leq \frac{\delta\xi(1-\gamma)^2}{4|\mathcal{S}||\mathcal{A}|^2}$. Thus we have $|\widehat{Q}_\mathrm{p}^*-\widetilde{Q}_{\mathrm{p},u}^*|\leq\frac{\delta\xi(1-\gamma)}{4|\mathcal{S}||\mathcal{A}|^2}$. As with probability $1-\delta$, a gap of $\frac{\delta\xi(1-\gamma)}{2|\mathcal{S}||\mathcal{A}|^2}$ exist in $\widehat{Q}_\mathrm{p}^*$, by Lemma \ref{l6}, we have $\widehat{\pi}_\mathrm{p}^*=\widetilde{\pi}_{\mathrm{p},u}^*$. The proof for $\widehat{c}_{\mathrm{p}}(\mu^*)$ is the same.
\end{proof}

\begin{proof}[Proof of Lemma \ref{l10}]
First, we consider the unperturbed case.
\begin{align*}
0&\leq Q^*-Q^{\widehat{\mu}^*,*}\\
&=Q^*-Q^{\mu^*,\widehat{c}(\mu^*)}+Q^{\mu^*,\widehat{c}(\mu^*)}-\widehat{Q}^{\mu^*,*}+\widehat{Q}^{\mu^*,*}-\widehat{Q}^*+\widehat{Q}^*-\widehat{Q}^{\widehat{\mu}^*,c(\widehat{\mu}^*)}+\widehat{Q}^{\widehat{\mu}^*,c(\widehat{\mu}^*)}-Q^{\widehat{\mu}^*,*}\\
&\leq Q^{\mu^*,\widehat{c}(\mu^*)}-\widehat{Q}^{\mu^*,*}+\widehat{Q}^{\widehat{\mu}^*,c(\widehat{\mu}^*)}-Q^{\widehat{\mu}^*,*}.
\end{align*}

Then we show that the perturbation only induce  an error of $\frac{\xi}{1-\gamma}$.
\begin{align*}
    \left|Q^{\mu^*,\nu_\mathrm{p}^*}-Q_\mathrm{p}^{\mu^*,\nu_\mathrm{p}^*}\right|&=\left|(I-\gamma P^{\mu^*,\nu_\mathrm{p}^*})^{-1}(r-r_\mathrm{p})\right|\\
    &\leq \frac{|r-r_\mathrm{p}|}{1-\gamma}\\
    &\leq \frac{\xi}{1-\gamma}\\
\end{align*}

Similarly, we have $\left|Q^{\mu_\mathrm{p}^*,\nu^*}-Q_\mathrm{p}^{\mu_\mathrm{p}^*,\nu^*}\right|\leq\frac{\xi}{1-\gamma}$. As $Q^{\mu_\mathrm{p}^*,\nu^*}\leq Q^*\leq Q^{\mu^*,\nu_\mathrm{p}^*}$ and $Q_\mathrm{p}^{\mu_\mathrm{p}^*,\nu^*}\geq Q_\mathrm{p}^*\geq Q_\mathrm{p}^{\mu^*,\nu_\mathrm{p}^*}$, we have $|Q_\mathrm{p}^*-Q^*|\leq\frac{\xi}{1-\gamma}$. Thus $|Q_\mathrm{p}^*-Q^*|\leq\frac{\xi}{1-\gamma}$. Similarly, we have $|Q_\mathrm{p}^{\widehat{\mu}^*,*}-Q^{\widehat{\mu}^*,*}|\leq\frac{\xi}{1-\gamma}$. Combining all parts together, we have

\begin{align*}
    \left|Q^{\widehat{\mu}^*,*}-Q^*\right|&\leq \left|Q^{\widehat{\mu}^*,*}-Q_\mathrm{p}^{\widehat{\mu}^*,*}\right|+\left|Q^*-Q_\mathrm{p}^*\right|+\left|Q_\mathrm{p}^{\widehat{\mu}^*,*}-Q_\mathrm{p}^*\right|\\
    &\leq\frac{2\xi}{1-\gamma}+\left|Q_\mathrm{p}^{\mu_\mathrm{p}^*,\widehat{c}_\mathrm{p}(\mu_\mathrm{p}^*)}-\widehat{Q}_\mathrm{p}^{\mu_\mathrm{p}^*,*}\right|+\left|\widehat{Q}_\mathrm{p}^{\widehat{\mu}_\mathrm{p}^*,c_\mathrm{p}(\widehat{\mu}_\mathrm{p}^*)}-Q_\mathrm{p}^{\widehat{\mu}_\mathrm{p}^*,*}\right|\\
\end{align*}

Combining these inequalities and we can get the proof.

\end{proof}

\begin{proof}[Proof of Theorem \ref{th2}]

By Lemma \ref{l1}, we have
\begin{align*}
    Q_\mathrm{p}^{\mu_\mathrm{p}^*,\widehat{c}_\mathrm{p}(\mu_\mathrm{p}^*)}-\widehat{Q}_\mathrm{p}^{\mu_\mathrm{p}^*,*}=\gamma(I-\gamma\widehat{P}^{\mu_\mathrm{p}^*,\widehat{c}_\mathrm{p}(\mu_\mathrm{p}^*)})^{-1}(P-\widehat{P})V_\mathrm{p}^{\mu_\mathrm{p}^*,\widehat{c}_\mathrm{p}(\mu_\mathrm{p}^*)}\\
\end{align*}

By uniform bound, with probability $1-\delta$ for all $u\in U$ defined in Lemma \ref{l9}, we have

$$\left|(P(s,a)-\widehat{P}(s,a))V_\mathrm{p}^{\mu_\mathrm{p}^*,\widetilde{c}_{\mathrm{p},u}(\mu_\mathrm{p}^*)}\right|\leq\sqrt{2\log(4|U|/\delta)}\sqrt{\frac{Var_{s,a}(V_\mathrm{p}^{\mu_\mathrm{p}^*,\widetilde{c}_{\mathrm{p},u}(\mu_\mathrm{p}^*)})}{N/|\mathcal{S}||\mathcal{A}|}}+\frac{2\log(4|U|/\delta)}{3(1-\gamma)N/|\mathcal{S}||\mathcal{A}|}.$$

Condition on the event defined in Lemma \ref{l9}, with probability $1-2\delta$, we have
$$\left|(P(s,a)-\widehat{P}(s,a))V_\mathrm{p}^{\mu_\mathrm{p}^*,\widehat{c}_{\mathrm{p}}(\mu_\mathrm{p}^*)}\right|\leq\sqrt{2\log(4|U|/\delta)}\sqrt{\frac{Var_{s,a}(V_\mathrm{p}^{\mu_\mathrm{p}^*,\widehat{c}_{\mathrm{p}}(\mu_\mathrm{p}^*)})}{N/|\mathcal{S}||\mathcal{A}|}}+\frac{2\log(4|U|/\delta)}{3(1-\gamma)N/|\mathcal{S}||\mathcal{A}|}.$$

With Lemma \ref{l3}, we have
$$\left|Q_\mathrm{p}^{\mu_\mathrm{p}^*,\widehat{c}_\mathrm{p}(\mu_\mathrm{p}^*)}-\widehat{Q}_\mathrm{p}^{\mu_\mathrm{p}^*,*}\right|\leq\frac{512\log(4|U|/\delta)}{(1-\gamma)^3N/|\mathcal{S}||\mathcal{A}|}+\frac{2\log(4|U|/\delta)}{3(1-\gamma)^2N/|\mathcal{S}||\mathcal{A}|}$$

If we set $N=\frac{8096|\mathcal{S}||\mathcal{A}|}{(1-\gamma)^3\epsilon^2}\log(\frac{256|\mathcal{S}|^2|\mathcal{A}|}{(1-\gamma)^2\xi\delta})$, then with probability larger than $1-\delta/2$, we have 
$$\left|Q_\mathrm{p}^{\mu_\mathrm{p}^*,\widehat{c}_\mathrm{p}(\mu_\mathrm{p}^*)}-\widehat{Q}_\mathrm{p}^{\mu_\mathrm{p}^*,*}\right|\leq\frac{\epsilon}{4}.$$

Similarly, with probability larger than $1-\delta/2$, we have
$$\left|\widehat{Q}_\mathrm{p}^{\widehat{\mu}_\mathrm{p}^*,c_\mathrm{p}(\widehat{\mu}_\mathrm{p}^*)}-Q_\mathrm{p}^{\widehat{\mu}_\mathrm{p}^*,*}\right|\leq\frac{\epsilon}{4}.$$

Using Lemma \ref{l10}, with probability larger than $1-\delta$, we have

$$\left|Q^{\widehat{\mu}^*,*}-Q^*\right|\leq\frac{2\xi}{1-\gamma}+\frac{\epsilon}{2}$$
Setting $\xi=\frac{\epsilon(1-\gamma)}{2}$, we have $|Q^{\widehat{\mu}^*,*}-Q^*|\leq \epsilon$. Similarly we can prove $|Q^{*,\widehat{\nu}^*}-Q^*|\leq \epsilon$. Finally we have that if $N\geq\frac{8096|\mathcal{S}||\mathcal{A}|}{(1-\gamma)^3\epsilon^2}\log(\frac{256|\mathcal{S}|^2|\mathcal{A}|}{(1-\gamma)^3\epsilon\delta})$, $\widehat{\pi}_\mathrm{p}^*=(\widehat{\mu}_\mathrm{p}^*,\widehat{\nu}_\mathrm{p}^*)$ is an $\epsilon$-Nash equilibrium strategy.

\end{proof}

\begin{proof}[Proof of Theorem \ref{th3}]

Set $\epsilon=\frac{\Delta}{2}$ and by Theorem \ref{th2}, we have that if $N=\frac{C}{(1-\gamma)^{3}\Delta^{2}}|\mathcal{S}||\mathcal{A}|\log(\frac{|\mathcal{S}||\mathcal{A}|}{(1-\gamma)\delta\epsilon})$, then $\widehat{\pi}^*$ is an $\frac{\Delta}{2}$-optimal strategy, which means $$|Q^*-Q^{\widehat{\pi}^*}|\leq \frac{\Delta}{2}.$$ By the definition of suboptimality gap, we have

$$\forall s\in \mathcal{S}_{\mathrm{max}},a\neq\mu^*(s): Q^*(s,\mu^*(s))-Q^*(s,a)\geq\Delta,$$
$$\forall s\in\mathcal{S}_{\mathrm{min}},a\neq\nu^*(s): Q^*(s,\nu^*(s))-Q^*(s,a)\leq -\Delta.$$

As $|Q^*-Q^{\widehat{\pi}^*}|\leq \frac{\Delta}{2}$, we have

$$\forall s\in \mathcal{S}_{\mathrm{max}},a\neq\mu^*(s): Q^{\widehat{\pi}^*}(s,\mu^*(s))-Q^{\widehat{\pi}^*}(s,a)\geq 0,$$
$$\forall s\in\mathcal{S}_{\mathrm{min}},a\neq\nu^*(s): Q^{\widehat{\pi}^*}(s,\nu^*(s))-Q^{\widehat{\pi}^*}(s,a)\leq 0,$$

which means ${\widehat{\pi}^*}=\pi^*$.

\end{proof}

\end{document}